\documentclass{article}
\usepackage[final]{nips_2016_name_removed} % NIPS 2016 style file

\usepackage[utf8]{inputenc} % allow utf-8 input
\usepackage[T1]{fontenc}    % use 8-bit T1 fonts
\usepackage{hyperref}       % hyperlinks
\usepackage{url}            % simple URL typesetting
\usepackage{booktabs}       % professional-quality tables
\usepackage{amsfonts}       % blackboard math symbols
\usepackage{nicefrac}       % compact symbols for 1/2, etc.
\usepackage{microtype}      % microtypography
\usepackage{amsmath,amssymb}
\usepackage{paralist}
\usepackage{tikz}
\usetikzlibrary{positioning,fit}
\usepackage{todonotes}
\usepackage{graphicx}
\usepackage{caption}
\usepackage{subcaption}
\usepackage[super]{nth}
\usepackage{refcount}

\usepackage{natbib}
\bibpunct{(}{)}{;}{a}{,}{,}

\usepackage{mythm}

\AtBeginDocument{

}

\usepackage[labelfont=bf]{caption}

\newcommand{\env}{\mu}

\newcommand{\policy}{\pi}
\newcommand{\pol}{\policy}

\newcommand{\POMDP}{\env} % the environment
\newcommand{\POMDPnoR}{{POMDP$\setminus$R}} % the notation for POMDP\R

\newcommand{\St}{\mathcal{S}}
\newcommand{\Ac}{\mathcal{A}}
\newcommand{\Ob}{\mathcal{O}}

\newcommand{\Pol}{\Pi}

\title{Counterfactual equivalence for POMDPs, and underlying deterministic environments}
\author{
  Stuart Armstrong \\
}
\date{\today}
\author{
  Stuart Armstrong \\
  Future of Humanity Institute \\
  University of Oxford \\
  Machine Intelligence Research Institute \\
  \texttt{stuart@philosophy.ox.ac.uk}
}

\begin{document}

%this comment serves no purpose

\maketitle

\begin{abstract}%
Partially Observable Markov Decision Processes (POMDPs) are rich environments often used in machine learning.
But the issue of information and causal structures in POMDPs has been relatively little studied.
This paper presents the concepts of equivalent and counterfactually equivalent POMDPs, where agents cannot distinguish which environment they are in though any observations and actions.
It shows that any POMDP is counterfactually equivalent, for any finite number of turns, to a deterministic POMDP with all uncertainty concentrated into the initial state.
This construction has a universality property, in that all such deterministic POMDPs define the same `pure' learning processes.
This allows a better understanding of POMDP uncertainty, information, and learning.
\end{abstract}

\section{Introduction}

Markov decision processes (MDPs) and Partially Observable Markov Decision Processes (POMDPs) \citep{sutton1998reinforcement,kaelbling1998planning} are useful and common tools in machine learning, with artificial agents evolving in these environments, generally seeking to maximise a reward.

But though there has been a lot of work on POMDPs from the practical perspective, there has been relatively little from the theoretical perspective.
This paper aims to partially fill that hole.
It first looks at notions of equivalence in POMDPs: two such structures are equivalent when an agent cannot distinguish which is which from any actions and observations it takes and makes.

A stronger notion is that of counterfactual equivalence; here, multiple agents sharing the same structure cannot distinguish it from another though any combinations of actions and observations.

Given these notions, this paper demonstrates that any POMDP will be counterfactually equivalent to a deterministic POMDP for any number $m$ of interaction terms.
A deterministic POMDP is one who transition and observation functions are deterministic, and hence all the uncertainty is concentrated in the initial state.

Having uncertainty expressed in this way allows one to clarify POMDPs from an information perspective: what can the agent be said to learn as it evolves in the POMDP, what it can change and what it can't.
Since the rest of the POMDP is deterministic, an agent that knows the environment can only gain information about the initial state.

This construction has a universality property, in that all such deterministic POMDPs define the same pure learning processes, where a pure learning process is one that can be decomposed as sums of knowledges about the initial state.

This allows better analysis of the causality in the POMDPs, using concepts that were initially designed for environments with the causal structure more naturally encoded, such as causal graphs \citep{causality}.

\section{Setup and notation}

The reward function in a POMDP is not important, as the focus of this paper is on its causal structure, with the reward just a component of the observation.

Thus define a \emph{partially observable Markov decision process without reward function (\POMDPnoR)}
$\POMDP = (\mathcal{S}, \mathcal{A}, \mathcal{O}, T, O, T_0)$~\citep{Choi11},
which consists of%
\begin{itemize}
\item a finite set of states $\mathcal{S}$,
\item a finite set of actions $\mathcal{A}$,
\item a finite set of observations $\mathcal{O}$,
\item a transition probability distribution $T: \mathcal{S} \times \mathcal{A} \to \Delta\mathcal{S}$ (where $\Delta\mathcal{S}$ is the set of probability distributions on $\mathcal{S}$),
\item a probability distribution $T_0 \in \Delta \mathcal{S}$ over the initial state $s_0$,
%specifying the probability to transition to state $s_{t+1}$ when taking action $a_t$ in state $s_t$,
\item an observation probability distribution $O: \mathcal{S} \to \Delta\mathcal{O}$.
%specifying the probability to receive observation $o_t$ when taking action $a_t$ in state $s_t$, and
\end{itemize}
%\paragraph{Agent-Environment Interaction.}
This \POMDPnoR\ will often be referred to as an environment (though \citet{hadfield2017inverse} refers to similar structures as world models).

The agent interacts with the environment in cycles:
initially, the environment is in state $s_0$ (given by $T_0$), and the agent receives observation $o_0$.
At time step $t$, the environment is in state $s_{t-1} \in \mathcal{S}$ and
the agent chooses an action $a_t \in \mathcal{A}$.
Subsequently
the environment transitions to a new state $s_t \in \mathcal{S}$ drawn from the distribution $T(s_t \mid s_{t-1}, a_t)$ and
the agent then receives an observation $o_t \in \mathcal{O}$ drawn from the distribution $O(o_t \mid s_t)$.
The underlying states $s_{t-1}$ and $s_t$ are not directly observed by the agent.

A history $h_t = o_0 a_1 o_1 \allowbreak a_2 o_2 \ldots a_t o_t$ is a sequence of actions and observations.
We denote the set of all observed histories of length $t$ with $\mathcal{H}_t := (\mathcal{A} \times \mathcal{O})^t$, and by $\mathcal{H}$ the set of all histories.

For $t'>t$, let $a_{t:t'}$ be the sequence of actions $a_{t}a_{t+1}\ldots a_{t'}$,
let $o_{t:t'}$ be the sequence of observations $o_{t}o_{t+1}\ldots o_{t'}$, and
let $s_{t:t'}$ the sequence of states $s_{t} s_{t+1}\ldots s_{t'}$.
Write $h_{t}\leq h_{t'}$ if $h_t=h_{t'}$ or if $h_{t'}=h_ta_{t+1}o_{t+1}\ldots a_{t'}o_{t'}$.

The set $\Pol$ is the set of \emph{policies}, functions $\pol: \mathcal{H} \to \Delta\mathcal{A}$ mapping histories to probability distributions over actions. Given a policy $\pol$ and environment $\env$,
we get a probability distribution over histories:
\begin{align*}
\env(o_0 a_1 o_1 \ldots a_t o_t \mid \pol) :=
\sum_{s_{0:t} \in \mathcal{S}^t} T_0(s_0)\prod_{k=1}^t O(o_k \mid s_k) T(s_k \mid s_{k-1}, a_k) \pol(a_k \mid a_1 o_1 \ldots a_{k-1} o_k).
\end{align*}
Since $\env$ gives the probabilities of everything except actions, and $\pol$ gives the probabilities of actions, all conditional probabilities between histories, actions, states, and so on, can be computed using $\env$, $\pol$, and Bayes' rule.
For instance, let $h_t$, $s$, and $\pol$ be such that $\env(h_t | s_0=s,\pol) \neq 0$.
Then by Bayes's rule:
\begin{align*}
\env(s_0=s| h_t,\pol) &= \frac{\env(h_t | s_0=s,\pol)\env(s_0=s|\pol)}{\sum_{s'\in\mathcal{S}}\env(h_t|s_0=s',\pol)\env(s_0=s'|\pol)}.
\end{align*}
Then note that $s_0$ is obviously independent of $\pol$, so this can be rewritten as
\begin{align*}
\env(s_0=s| h_t) &= \frac{\env(h_t | s_0=s,\pol)\env(s_0=s)}{\sum_{s'\in\mathcal{S}}\env(h_t|s_0=s',\pol)\env(s_0=s')},
\end{align*}
which can be computed from $\env$. In the case where there exists no $\pol$ with $\env(h_t | s_0=s,\pol) \neq 0$, set $\env(s_0=s| h_t)$ to $0$.

\section{Equivalence and counterfactual equivalence}

\begin{definition}[Similarity]
The environments $\env$ and $\env^*$ are (observationally) similar if they have the same sets $\mathcal{A}$, and $\mathcal{O}$.
Consequently, they have the same sets of histories $\mathcal{H}$, and hence the same sets of policies $\Pol$.
\end{definition}

\subsection{Equivalence}

We'll say that two environments $\env$ and $\env^*$ are $m$-equivalent if an agent in one cannot figure out which one it is in during the first $m$ turns.

To formalise this:
\begin{definition}[Equivalence]
The environments $\env$ and $\env^*$ are $m$-equivalent if they are similar (and hence have the same sets of histories), and, for all $h_t, h_{t'}\in\mathcal{H}=\mathcal{H}^*$ with $t,t'\leq m$ and all policies $\pol\in\Pol=\Pol^*$,
\begin{align}\label{equiv:eq}
\env(h_{t'} | h_t, \pol) = \env^*(h_{t'} | h_t, \pol). 
\end{align}
If they are $m$-equivalent for all $m$, they are equivalent.
\end{definition}

\subsection{Counterfactual equivalence}
We'll say that two environments $\env$ and $\env^*$ are $m$-counterfactually equivalent if multiple agents sharing the same environment, cannot figure out which one they are in during the first $m$ turns.

This is a bit more tricky to define; in what sense can multiple agents be said to share the same environment?
One idea is that if two agents are in the same state and choose the same action, they will then move together to the same next state (and make the same next observation).

To formalise this, define:
\begin{definition}[Environment policy]
The $\pol_\env$ is a deterministic environment policy of length $m$, if it is triplet $(\widehat{T}_0,\widehat{T},\widehat{O})$, where $\widehat{T}_0\in\St$, $\widehat{O}:\St\times\{0,\ldots,m\}\to\Ob$, and $\widehat{T}: \St\times\Ac\times\{1,\ldots,m\} \to \St$.
Let $\Pol_\env^m$ be the set of all environment policies of length $m$.
\end{definition}

The idea is that $(\widehat{T}_0,\widehat{T},\widehat{O})$ contain all information as to how the stochasticity in $T_0$, $T$, and $O$ are resolved in the environment. The $\widehat{T}_0$ gives a single initial state, $\widehat{T}(s,a,i)=s'$ means that an agent in state $s$ on turn $i$, taking action $a$, will move to state $s'$, and $\widehat{O}(s,i)=o$ means that an agent arriving in state $s$ on turn $i$ will make observation $o$.

The environment $\env$ gives a distribution over elements of $\Pol_\env^m$:
\begin{align}\label{prob:env:policy}
\env(\widehat{T}_0,\widehat{T},\widehat{O}) = T_0(\widehat{T}_0) \cdot\left[ \prod_{s\in\St,a\in\Ac,1\leq i\leq m} T\left(\widehat{T}(s,a,i) \middle| s, a\right) \prod_{s\in\St,0\leq i\leq m} O\left(\widehat{O}(s,i)\middle|s\right)\right].
\end{align}

For the first $m$ turns of interaction with the environment, the agent can either see itself as updating using $T_0$, $T$, and $O$, or it can see itself as following a deterministic environment policy $\pol_\env$, chosen according to the above probability.

Given an environment policy and an actual policy, the probability of a certain history can be computed. If $\pol\in\Pol$ is deterministic, $\env(h_t | \pol_\env, \pol)$ will be always either $1$ or $0$, since $\pol_\env$ and $\pol$ deterministically determine all the states, observations, and actions.

Using $\env$ and Bayes's rule, this conditional probability can be inverted to compute $\env(\pol_\env| h_t, \pol)$, which is $\env(\pol_\env| h_t)$ since $\pol_\env$ and $\pol$ are independent of each other.

So this gives a formalisation of what it means to have several agents sharing the same environment: they share an environment policy.
\begin{definition}[Counterfactual equivalence]
The environments $\env$ and $\env^*$ are $m$-counterfactually equivalent if they are similar, and if for any collection $(h_{t_i},\pol_i)_{i\leq n}$ of pairs of histories and policies with $t_i\leq m$,
\begin{align}\label{con:equiv:eq}
\sum_{\pol_\env \in \Pol^m_\env} \env(\pol_\env) \prod_{i\leq n} \env(h_{t_i}|\pol_\env,\pol_i)
=
\sum_{\pol_{\env^*} \in \Pol^m_{\env^*}} \env^*(\pol_{\env}^*) \prod_{i\leq n} \env^*(h_{t_i}|\pol_{\env^*},\pol_i)
\end{align}
If they are $m$-counterfactually equivalent for all $m$, they are counterfactually equivalent.
\end{definition}
The terms in \autoref{con:equiv:eq} are the joint probabilities of $n$ agents, using policies $\pol_i$ and sharing the same environment policy, each seeing the histories $h_{t_i}$.

And finally:
\begin{definition}
If $\env$ and $\env^*$ are $m$-equivalent (or $m$-counterfactually equivalent) for all $m$, they are equivalent (or counterfactually equivalent).
\end{definition}

A useful result is:
\begin{proposition}
If the environments $\env$ and $\env^*$ are $m$-counterfactually equivalent, then they are $m$-equivalent.
\end{proposition}
\begin{proof}
If $h_t \nleq h_{t'}$, then
\begin{align*}
\env(h_{t'} | h_t, \pol) = 0 = \env^*(h_{t'} | h_t, \pol),
\end{align*}
since $h_{t'}$ is impossible, given $h_t$.

If $h_t \leq h_{t'}$, then
\begin{align*}
\env(h_{t'} | h_t, \pol) = \frac{\env(h_{t'},h_t | \pol)}{\env(h_t|\pol)} = \frac{\env(h_{t'} | \pol)}{\env(h_t|\pol)}.
\end{align*}

For the counterfactually equivalent $\env$ and $\env^*$, the case of $n=1$, $(h_t|\pol)$, demonstrates that $\env(h_{t}|\pol)=\env^*(h_{t}|\pol)$.
The same argument shows $\env(h_{t'}|\pol)=\env^*(h_{t'}|\pol)$, demonstrating $\env(h_{t'} | h_t, \pol) = \env^*(h_{t'} | h_t, \pol)$ and establishing \autoref{equiv:eq}.

\end{proof}

\section{Examples}\label{examples}

Consider the $\env$ of \autoref{env:mu}.
This has $\St=\{s_0,s^{00},s^{01},s^{10},s^{11}\}$, $\Ac=\{a^0,a^1\}$, and $\Ob=\St$. Since the observations and states are the same, with trivial $O$, this is actually a Markov Decision process \citep{sutton1998reinforcement}.
The agent starts in $s_0$, chooses between two actions, an each action leads separately to one of two outcomes, with equal probability.

\begin{figure}[htbp]
\begin{center}
\begin{tikzpicture}[node distance=4cm,
  thick,main node/.style={circle,fill=blue!20,draw,font=\sffamily\small},action node/.style={circle,fill=red!20,draw,font=\sffamily\small}]

  \node[main node] (1) {$s_0$};
  \node[action node] at (1.5,0.75) (2) {$a^{0}$};
  \node[action node] at (1.5,-0.75) (3)  {$a^1$};
  \node[main node] at (3.5,2.25) (4) {$s^{00}$};
  \node[main node] at (3.5,0.75) (5) {$s^{01}$};
  \node[main node] at (3.5,-0.75) (6) {$s^{10}$};
  \node[main node] at (3.5,-2.25) (7) {$s^{11}$};

  \path[every node/.style={font=\sffamily\small}]
	(1) edge[->] (2)
    	edge[->] (3)
    (2) edge[->] node [above left] {$1/2$} (4)
    	edge[->] node [below] {$1/2$} (5)
    (3) edge[->] node [above] {$1/2$} (6)
    	edge[->] node [below left] {$1/2$} (7)
    ;
\end{tikzpicture}
\end{center}
\caption{Environment $\env$: two choices, four outcomes}
\label{env:mu}
\end{figure}
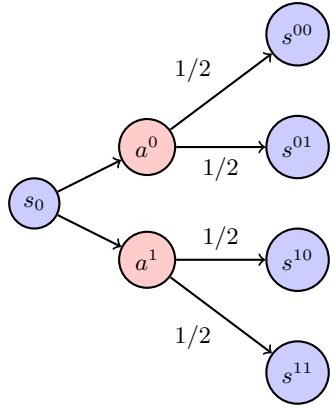

Compare with $\env'$ of \autoref{env:mu:p}.
The actions and observations are the same (hence the two environments are similar), but the state set is larger.
Instead of one initial state, there are two, $s_0^0$ and $s_0^1$, leading to the same observation $o_0$.
These two states are equally likely under $T_0$, and lead deterministically to different states if the agent chooses $a^0$.

It's not hard to see that $\env$ and $\env'$ are counterfactually equivalent.
The environment $\env'$ has just shifted the uncertainty about the result of $a^0$, out of $T$ and into the initial distribution $T_0$.

\begin{figure}[htbp]
\begin{center}
\begin{tikzpicture}[node distance=4cm,
  thick,main node/.style={circle,fill=blue!20,draw,font=\sffamily\small},action node/.style={circle,fill=red!20,draw,font=\sffamily\small}]

  \node[main node, label={[shift={(-0.3,0.0)}]left:$T_0(s^0_0)=1/2$}] at (0,0.75) (0) {$s_0^0$};
  \node[main node, label={[shift={(-0.3,0.0)}]left:$T_0(s^1_0)=1/2$}] at (0,-0.75) (1) {$s_0^1$};
  \node[action node] at (1.5,0.75) (3) {$a^0$};
  \node[action node] at (1.5,2.25) (2) {$a^0$};
  \node[action node] at (1.5,-0.75) (2r)  {$a^1$};
  \node[action node] at (1.5,-2.25) (3r)  {$a^1$};
  \node[main node] at (3.5,2.25) (4) {$s^{00}$};
  \node[main node] at (3.5,0.75) (5) {$s^{01}$};
  \node[main node] at (3.5,-0.75) (6) {$s^{10}$};
  \node[main node] at (3.5,-2.25) (7) {$s^{11}$};

  \node[draw,dotted,fit=(0) (1), label=above:{$\mathbf{o_0}$}] {};

  \path[every node/.style={font=\sffamily\small}]
	(0) edge[->] (2)
    	edge[->] (2r)
	(1) edge[->] (3)
    	edge[->] (3r)
    (2) edge[->] node [above] {$1$} (4)
    (2r) edge[->] node [above] {$1/2$} (6)
    	 edge[->] node [above] {$1/2$} (7)
    (3) edge[->] node [below] {$1$} (5)
    (3r) edge[->] node [below] {$1/2$} (7)
    	 edge[->] node [below] {$1/2$} (6)
    ;
\end{tikzpicture}
\end{center}
\caption{Environment $\env'$: two initial states, two choices, four outcomes.}
\label{env:mu:p}
\end{figure}
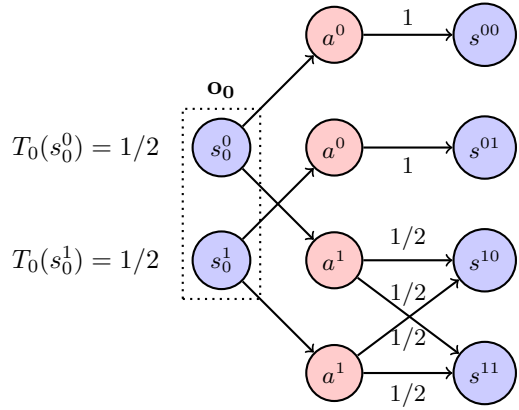

Contrast both of these with the environment $\env''$ of \autoref{env:mu:pp}, which has a the same $\St$, $\Ac$, $\Ob$, $T_0$, and $O$ as $\env'$, but different behaviour under action $a^1$ (hence a different $T$).

It's not hard to see that all three environments are equivalent: given history $o_0a^i$, the agent is equally likely to end up in state $s^{i0}$ and $s^{i1}$, and that's the end of the process.

They are not, however, counterfactually equivalent.
There are four environment policies in $\env$ (and in $\env'$) of non-zero probability.
They can be labeled $\pol_{ij}$, which sends $a^0$ to $s^{0i}$ and $a^1$ to $s^{1j}$.
Each one has probability $1/4$.

There are two environment policies in $\env''$ of non-zero probability; they can be labeled $\pol_i$, which simply chooses the starting state $s_0^i$.
Each one has probability $1/2$.

Since there are only two actions and they are only used once, the policies of these environments can be labeled by that action.

Then consider the two pairs of policies and histories $(a^0, o_0a^0o^{00})$ and $(a^1, o_0a^1o^{11})$.
Under the environment policy $\pol_{01}$, both these pairs are certainly possible, so they have an non-zero probability under $\env$ (and $\env'$).
However, $(a^0, o_0a^0o^{00})$ is impossible under $\pol_1$, while $(a^1, o_0a^1o^{11})$ is impossible under $\pol_0$.
So there are no environment policies in $\env''$ that make both those histories possible.
Thus $\env''$ is not counterfactually equivalent to $\env$ and $\env'$.

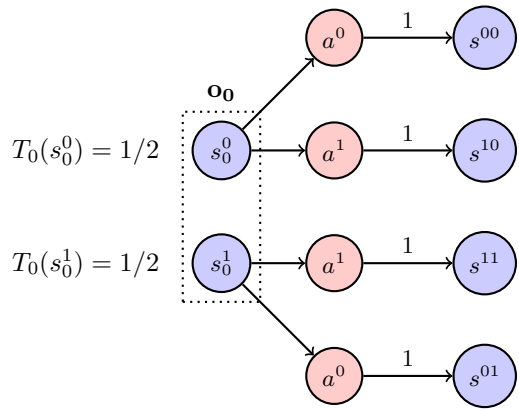
\begin{figure}[htbp]
\begin{center}
\begin{tikzpicture}[node distance=4cm,
  thick,main node/.style={circle,fill=blue!20,draw,font=\sffamily\small},action node/.style={circle,fill=red!20,draw,font=\sffamily\small}]

  \node[main node, label={[shift={(-0.3,0.0)}]left:$T_0(s^0_0)=1/2$}] at (0,0.75) (0) {$s_0^0$};
  \node[main node, label={[shift={(-0.3,0.0)}]left:$T_0(s^1_0)=1/2$}] at (0,-0.75) (1) {$s_0^1$};
  \node[action node] at (1.5,2.25) (a00) {$a^0$};
  \node[action node] at (1.5,0.75) (a10) {$a^1$};
  \node[action node] at (1.5,-0.75) (a11)  {$a^1$};
  \node[action node] at (1.5,-2.25) (a01)  {$a^0$};
  \node[main node] at (3.5,2.25) (s00) {$s^{00}$};
  \node[main node] at (3.5,0.75) (s10) {$s^{10}$};
  \node[main node] at (3.5,-0.75) (s11) {$s^{11}$};
  \node[main node] at (3.5,-2.25) (s01) {$s^{01}$};

  \node[draw,dotted,fit=(0) (1), label=above:{$\mathbf{o_0}$}] {};

  \path[every node/.style={font=\sffamily\small}]
	(0) edge[->] (a00)
    	edge[->] (a10)
	(1) edge[->] (a11)
    	edge[->] (a01)
    (a00) edge[->] node [above] {$1$} (s00)
    (a10) edge[->] node [above] {$1$} (s10)
    (a11) edge[->] node [above] {$1$} (s11)
    (a01) edge[->] node [above] {$1$} (s01)
    ;
\end{tikzpicture}
\end{center}
\caption{Environment $\env''$: two initial states, two choices, four counterfactually correlated outcomes.}
\label{env:mu:pp}
\end{figure}

\section{Underlying deterministic environment}

In this section, the environment $\env^*$ is assumed to have all its special features indicated by a $*$ -- so it will have state space $\St^*$, transitions function $T_0^*$, and so on.

The main result is:
\begin{theorem}\label{determine:env}
For $m$ and all environments $\env$, there exists an environment $\env^*$ that is $m$-counterfactually equivalent to $\env$, and on which the transition function $T^*$, and the observation function $O^*$, are both deterministic.
\end{theorem}
\begin{proof}
Let $\Ac^*=\Ac$ and $\Ob^*=\Ob$, so $\env$ and $\env^*$ are similar.

Define $\mathcal{S}^*=\mathcal{S}\times\Pol_\env^m\times\{0,\ldots,m\}$.

Recall that any $\pol_\env\in\Pi_\env^m$ decomposes as $(\widehat{T}_0,\widehat{T},\widehat{O})$.
The deterministic $O^*$ is defined as sending the state $(s,\pol_\env,i)$ to $\widehat{O}(s,i)$.
The deterministic $T^*$ is defined as mapping $(s,\pol_\env,i)$ and the action $a$ to $(\widehat{T}(s,a,i),\pol_\env,i+1)$.
For the rest of the proof, we'll see $T^*$ and $O^*$ as functions, mapping into $\St^*$ and $\Ob^*=\Ob$.

The initial distribution $T_0^*(s,\pol_\env,i)$ is $\env(\pol_\env)$ if $s=\widehat{T}_0$ and $i=0$, and is $0$ otherwise.
This defines $\env^*$.

We now need to show that $\env$ and $\env^*$ are $m$-counterfactually equivalent.
The proof is not conceptually difficult, one just has to pay careful attention to the notation.

Let $Q^*\subset\Pol^m_{\env^*}$ be defined as the elements $\pol_{\env^*}$ of the form\footnote{
Ignoring the extra variable: for all $i$, $T^*(s^*,a,i):=T^*(s^*,a)$ and $O^*(s^*,i):=O^*(s^*)$.
} $((\widehat{T}_0,\pol_\env,0),T^*,O^*)$, for $\widehat{T}_0$ given by $\pol_\env$.
Let $f$ be the (bijective) map taking such $\pol_{\env^*}$ to the corresponding $\pol_\env$.

Since $T^*$ and $O^*$ are deterministic, \autoref{prob:env:policy} and the definition of $T_0^*$ imply that $\env^*(\pol_{\env^*})=0$ if $\pol_{\env^*}\notin Q^*$.
Again by the definition of $T_0^*$:
\begin{align}\label{eq:one}
\env^*(\pol_{\env^*}) = \env(f(\pol_{\env^*})).
\end{align}

Then note that $T^*$ preserves the middle component of $(s,\pol_\env,i)$.
Given a state $(s,\pol_\env,i)$ and an action $a$, the next state and observation in $\env^*$ will be given by $(\widehat{T}(s,a,i),\pol_\env,i+1)$ and $\widehat{O}(s,i+1)$.
Similarly, given the state $s$, environmental policy $\pol_\env$, and action $a$, the next state and observation in $\env$ will be given by $\widehat{T}(s,a,i)$ and $\widehat{O}(s,i+1)$.

So an agent in $\env^*$, starting in $(\widehat{T}_0,\pol_\env,0)$, and an agent in $\env$, having environment policy $\pol_\env$ (and hence starting in $\widehat{T}_0$), would, if they chose the same actions, see the same observations.
Now, `starting in $(\widehat{T}_0,\pol_\env,0)$' can be rephrased as `having environment policy $f^{-1}(\pol_\env)$'.
Since the policies of the agent are dependent on actions and observations only, this means that for all $h_t\in\mathcal{H}$ and $\pol\in\Pol$:
\begin{align}\label{eq:two}
\env^*(h_t | \pol_{\env^*}, \pol) = \env(h_t | f(\pol_{\env^*}), \pol).
\end{align}

Together, \autoref{eq:one} and \autoref{eq:two} give the desired equality of \autoref{con:equiv:eq}: for collections $(h_{t_i},\pi_i)_{i\leq n}$ of history-policy pairs with $t_i\leq m$,
\begin{align*}
\sum_{\pol_{\env^*} \in \Pol^m_{\env^*}} \env^*(\pol_{\env}^*) \prod_{i\leq n} \env^*(h_{t_i}|\pol_{\env^*},\pol_i)
&= \sum_{\pol_{\env^*} \in Q^*} \env^*(\pol_{\env^*}) \prod_{i\leq n} \env^*(h_{t_i}|\pol_{\env^*},\pol_i) \\
&= \sum_{\pol_{\env^*} \in Q^*} \env(f(\pol_{\env^*})) \prod_{i\leq n} \env(h_{t_i}|f(\pol_{\env^*}),\pol_i) \\
&= \sum_{\pol_{\env} \in \Pol^m_{\env}} \env(\pol_{\env}) \prod_{i\leq n} \env(h_{t_i}|\pol_{\env},\pol_i),
\end{align*}
since $f$ is a surjection onto $\Pol^m_{\env}$.

\end{proof}
In the above construction, all the uncertainty and stochasticity of the initial $\env$ has been concentrated into the distribution $T_0^*$ over the initial state $s_0^*$ in $\env^*$.

Note that though the construction will work for every $m$, the size of $\mathcal{S}^*$ increases with $m$, so the limit of this $\env^*$ as $m\to\infty$ has a countable infinite number of states, rather than a finite number.

\subsection{`Universality' of the underlying deterministic environment}

For many $\env$, much simpler constructions are possible.
See for instance environment $\env^*$ of \autoref{env:mu:s}.
It is deterministic in $O^*$ and $T^*$, and counterfactually equivalent to $\env$ and $\env'$ in \autoref{examples}.
But $\env$ has $5$ states and $10$ state-action pairs, so there are $5\times 5^{10}$ different environment policies\footnote{
Though only $4$ of non-zero probability.
}, meaning that $\mathcal{S}\times\Pol_\env^1\times\{0,1\}$ is of magnitude $5^{12}\times 2 = 488281250$, much larger than the $8$ states of $\env^*$.

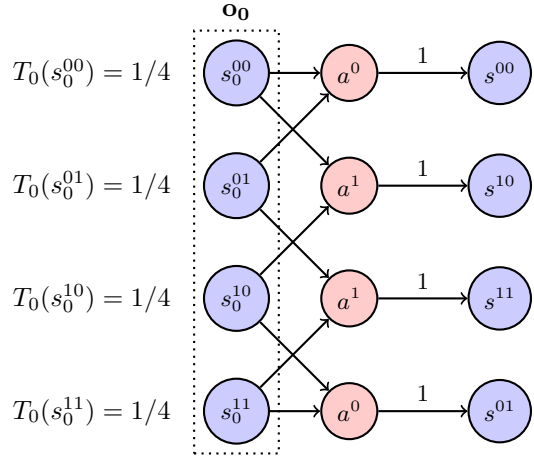
\begin{figure}[htbp]
\begin{center}
\begin{tikzpicture}[node distance=4cm,
  thick,main node/.style={circle,fill=blue!20,draw,font=\sffamily\small},action node/.style={circle,fill=red!20,draw,font=\sffamily\small}]

  \node[main node, label={[shift={(-0.3,0.0)}]left:$T_0(s_0^{00})=1/4$}] at (0,2.25) (00) {$s_0^{00}$};
  \node[main node, label={[shift={(-0.3,0.0)}]left:$T_0(s_0^{01})=1/4$}] at (0,0.75) (01) {$s_0^{01}$};
  \node[main node, label={[shift={(-0.3,0.0)}]left:$T_0(s_0^{10})=1/4$}] at (0,-0.75) (10) {$s_0^{10}$};
  \node[main node, label={[shift={(-0.3,0.0)}]left:$T_0(s_0^{11})=1/4$}] at (0,-2.25) (11) {$s_0^{11}$};
  \node[action node] at (1.5,2.25) (a00) {$a^0$};
  \node[action node] at (1.5,0.75) (a10) {$a^1$};
  \node[action node] at (1.5,-0.75) (a11)  {$a^1$};
  \node[action node] at (1.5,-2.25) (a01)  {$a^0$};
  \node[main node] at (3.5,2.25) (s00) {$s^{00}$};
  \node[main node] at (3.5,0.75) (s10) {$s^{10}$};
  \node[main node] at (3.5,-0.75) (s11) {$s^{11}$};
  \node[main node] at (3.5,-2.25) (s01) {$s^{01}$};

  \node[draw,dotted,fit=(00) (11), label=above:{$\mathbf{o_0}$}] {};

  \path[every node/.style={font=\sffamily\small}]
	(00) edge[->] (a00)
    	edge[->] (a10)
	(01) edge[->] (a00)
    	edge[->] (a11)
	(10) edge[->] (a01)
    	edge[->] (a10)
    (11) edge[->] (a01)
    	edge[->] (a11)
    (a00) edge[->] node [above] {$1$} (s00)
    (a01) edge[->] node [above] {$1$} (s01)
    (a10) edge[->] node [above] {$1$} (s10)
    (a11) edge[->] node [above] {$1$} (s11)
    ;
\end{tikzpicture}
\end{center}
\caption{Environment $\env^*$: four initial states, two choices, four outcomes.}
\label{env:mu:s}
\end{figure}

This poses the question, as to which deterministic POMDP is preferable to model the initial POMDP.
Fortunately, there is a level at which all counterfactually equivalent deterministic POMDPs are the same.

\begin{definition}[Pure learning process]
On $\env$, let $P: \mathcal{H}_{\leq m}\to[0,1]$ be a map from histories of length $m$ or less, to the unit interval.
Then $P$ is a pure learning process if there exists a deterministic $\env^*$, $m$-counterfactually equivalent to $\env$, such that $P$ can be expressed as
\begin{align}\label{univ:eq}
P(h_t)=\sum_{s\in\St^*} p_{s^*}\env^*(s_0^*=s^*|h_t),
\end{align}
for constants $p_{s^*}\in[0,1]$.
\end{definition}

These pure learning processes are seen to compute a probability over the stochastic elements of the environment.
Then the universality result is:
\begin{theorem}
Let $P$ be a pure learning process on $\env$, and let $\env^*$ be deterministic and $m$-counterfactually equivalent to $\env$.
Then there exists constants $p_{s^*}$ for $s\in\St$, such that $P$ can be defined as in \autoref{univ:eq}.
\end{theorem}
\begin{proof}
Since $P$ is a pure learning process, we already know that there exists a deterministic environment, $m$-counterfactually equivalent to $\env$, where $P$ decomposes as \autoref{univ:eq}.
Since being $m$-counferfactually equivalent is a transitive property, we may as well assume that $\env$ itself is this environment.

We now need to define the $p_{s^*}$ on $\env^*$, and show they generate the same $P$.

Let $\Pol_0\subset\Pol$ be the set of deterministic policies.
Since $\env$ is deterministic itself, apart from $T_0$, a choice of $s_0$ and a choice of $\pol\in\Pol_0$ determines a unique history $h_m$ of length $m$.
Therefore each $s\in \St$ defines a map $f_{s}:\Pol_0 \to \mathcal{H}_m$.
Define the subset $\mathcal{F}(f_s)$ as the set of all $s'$ such that $f_s=f_{s'}$; these subsets form a partition of $\mathcal{S}$.

Since $\env^*$ is also deterministic, its state space $\St^*$ has a similar partition.

Given an $f_s$, define the collection of pairs $(f_s(\pol_i),\pol_i)_{\pol_i\in\Pol_0}$.
For a deterministic environment, an environment policy of non-zero probability is just a choice of initial state.
So, writing $\env(s)$ for $\env(s_0=s)=T_0(s)$, \autoref{con:equiv:eq} with that collection of pairs becomes:
\begin{align}\label{det:con:equiv}
\sum_{s'\in\St} \env(s') \prod_{\pol_i\in\Pol_0} \env(f_s(\pol_i)|s',\pol_i) = \sum_{s^*\in\St^*} \env^*(s^*) \prod_{\pol_i\in\Pol_0} \env^*(f_{s^*}(\pol_i)|s^*,\pol_i).
\end{align}
Since everything is deterministic, the expression $\prod_{\pol_i\in\Pol_0} \env(f_s(\pol_i)|s',\pol_i)$ must be either $0$ or $1$, and it is $1$ only if $s'\in\mathcal{F}(f_s)$.
Thus \autoref{det:con:equiv} can be further rewritten as
\begin{align*}
\sum_{s'\in\mathcal{F}(f_s)} \env(s') \mathbf{1} = \sum_{s^*\in\mathcal{F}^*(f_s)} \env^*(s^*) \mathbf{1}.
\end{align*}
This demonstrates that the probability under $\env$ of any $\mathcal{F}(f_s)$, is the same as the probability under $\env^*$ of $\mathcal{F}^*(f_s)$; so, writing $\env(\mathcal{F}(f_s))$ for $\env(s_0\in\mathcal{F}(f_s)$,
\begin{align}\label{equality:eq}
\env(\mathcal{F}(f_s))=\env^*(\mathcal{F}^*(f_s)).
\end{align}

So for all $s^*\in\mathcal{F}^*(f_s)$, with $\env^*(s^*)\neq 0$, define
\begin{align*}
p_{s^*}=\frac{\sum_{s'\in\mathcal{F}(f_s)}\env(s')p_{s'}}{\env(\mathcal{F}(f_s))}.
\end{align*}
Thus $p_{s^*}$ for $s^*\in\mathcal{F}(f_s)$ is equal to the weighted average of $p_s$ in $\mathcal{F}(f_s)\subset\St$.
For the $s^*$ with $\env(s^*)=0$, set $p_{s^*}$ to any value.
This defines the $p_{s^*}$, and hence a $P^*$ on $\mathcal{H}$ via \autoref{univ:eq}.

We now need to show that $P=P^*$. Note first that
\begin{align}\label{note:eq}
\begin{aligned}
\sum_{s^*\in\mathcal{F}^*(f_s)} \env^*(s^*)p_{s^*} &= \frac{\env^*(\mathcal{F}^*(f_s))}{\env(\mathcal{F}(f_s))} \sum_{s'\in\mathcal{F}(f_s)}\env(s')p_{s'} \\
&=\sum_{s'\in\mathcal{F}(f_s)}\env(s')p_{s'}.
\end{aligned}
\end{align}

Now let $h_t$ be a history with $t\leq m$, and $\pol$ any deterministic policy that, upon given an initial segment $h_{t'} < h_t$, will generate the action $a_{t'+1}$.
Thus $\pol$ is a policy that could allow $h_t$ to happen.

Let $\St(h_t)$ be the set of all $s\in\St$ such that $h_t \leq f_s(\pol)\in\mathcal{H}_m$.
This means that, if the agent started in $s$ and followed $\pol$, it would generate a history containing $h_t$ -- hence that it would generate $h_t$.

That set can be written as a union $\St(h_t)=\bigcup_{h_t \leq f_s(\pol)}\mathcal{F}(f_s)$.
The observation of $h_t$ is thus equivalent to $s_0\in\St(h_t)$.
Consequently
\begin{align*}
P(h_t)&=\frac{\sum_{s'\in\St(h_t)}\env(s')p_s}{\env(\St(h_t))}\\
&=\frac{\sum_{f_s: h_t \leq f_s(\pol)}\sum_{s'\in\mathcal{F}(f_s)} \ \env(s')p_s}{\sum_{f_s: h_t \leq f_s(\pol)} \ \env(\mathcal{F}(f_s))}\\
&= \frac{\sum_{f_s: h_t \leq f_s(\pol)}\sum_{s^*\in\mathcal{F}^*(f_s)} \ \env^*(s^*)p_{s^*}}{\sum_{f_s: h_t \leq f_s(\pol)} \ \env^*(\mathcal{F}(f_s))}\\
&= P^*(h_t),
\end{align*}
by \autoref{equality:eq} and \autoref{note:eq}.

\end{proof}

Thus any deterministic $m$-counterfactually equivalent environment can be used to define any pure learning process: they are all interchangeable for this purpose.

%%%%%%%%%%%%%%%%%%%%%%%%%%%%%%%%%%%%%%%%%%%%%%%%%%%%%%%%%%%%%%%

\bibliography{ref}

\end{document}